\documentclass[preprint,12pt]{elsarticle}

\usepackage{amssymb}
\usepackage[utf8]{inputenc} 
\usepackage[T1]{fontenc}    
\usepackage{hyperref}       
\usepackage{url}            
\usepackage{booktabs}       
\usepackage{amsfonts}       
\usepackage{nicefrac}       
\usepackage{microtype}      

\usepackage{amsmath}
\usepackage{wrapfig}
\usepackage[ruled,noline]{algorithm2e}
\usepackage{amsthm}
\usepackage{algorithmic}
\usepackage{enumitem, tabularx}
\usepackage{color,soul}
\usepackage{mathtools}
\usepackage{varwidth}
\usepackage{multirow}
\usepackage{diagbox}
\usepackage{setspace}

\newtheorem{theorem}{Theorem}
\newtheorem{lemma}{Lemma}

\DeclareMathOperator*{\argmin}{argmin}

\journal{Neural Networks}

\begin{document}

\begin{frontmatter}

\title{Schematic Memory Persistence and Transience for Efficient and Robust Continual Learning}

\author[1]{Yuyang Gao}
\ead{yuyang.gao@emory.edu}

\author[2]{Giorgio A. Ascoli}
\ead{ascoli@gmu.edu}

\author[1]{Liang Zhao\corref{cor1}}
\ead{liang.zhao@emory.edu}

\cortext[cor1]{Corresponding author}

\address[1]{Department of Computer Science, Emory University, Atlanta, GA, United States}
\address[2]{Center for Neural Informatics, Bioengineering Department, and Krasnow Institute for Advanced Study, George Mason University, Fairfax, VA, United States}

\begin{abstract}
Continual learning is considered a promising step towards next-generation Artificial Intelligence (AI), where deep neural networks (DNNs) make decisions by continuously learning a sequence of different tasks akin to human learning processes. It is still quite primitive, with existing works focusing primarily on avoiding (catastrophic) forgetting. However, since forgetting is inevitable given bounded memory and unbounded task loads, ‘how to reasonably forget’ is a problem continual learning must address in order to reduce the performance gap between AIs and humans, in terms of 1) memory efficiency, 2) generalizability, and 3) robustness when dealing with noisy data. To address this, we propose a novel ScheMAtic memory peRsistence and Transience (SMART) \footnote{ Code available at: \scriptsize \url{https://drive.google.com/file/d/1OfSbvKq488wq1HyNsWX5tbHyvmiqPPv9}.} framework for continual learning with external memory that builds on recent advances in neuroscience. The efficiency and generalizability are enhanced by a novel long-term forgetting mechanism and schematic memory, using sparsity and ‘backward positive transfer’ constraints with theoretical guarantees on the error bound. Robust enhancement is achieved using a novel short-term forgetting mechanism inspired by background information-gated learning. Finally, an extensive experimental analysis on both benchmark and real-world datasets demonstrates the effectiveness and efficiency of our model.

\end{abstract}




\end{frontmatter}


\section{Introduction}
An important step toward next-generation Artificial Intelligence (AI) (i.e., artificial general intelligence) is a promising new domain known as continual learning, where deep neural networks (DNNs) make decisions by continuously learning a sequence of different tasks , similar to the way humans learn \cite{parisi2019continual}.
Compared with the extensive research on traditional AI for learning isolated tasks individually, continual learning is still in its very primitive stage \cite{parisi2019continual}. At present, the primary goal is essentially to avoid (catastrophic) forgetting of previously learned tasks when an agent is learning new tasks. This important issue is receiving a great deal of attention, with researchers proposing a number of different strategies which can be categorized into three main approaches. The first focuses on simply restricting changes in the magnitudes of all the model parameters \cite{benna2016computational, kirkpatrick2017overcoming,zenke2017continual, aljundi2018memory}.
However, given the limited amount of neural resources available, this approach necessarily sacrifices the flexibility needed to fit new tasks by protecting old tasks \cite{parisi2019continual}. 
An alternative approach used by dynamic architecture-based methods is to allow the model to incrementally expand the network whenever new tasks arrive. Constantly increasing the number of neurons is highly inefficient, however, and very different to the way biological neural networks function \cite{french1994dynamically, rusu2016progressive, mallya2018packnet}.
The third approach, replay-based methods, avoid this problem by storing the historical training data with the help of external memory \cite{lopez2017gradient, chaudhry2019continual, isele2018selective, aljundi2019gradient}.

Among these three options, the replay-based methods are arguably more effective in terms of performance and bio-inspiration \cite{robins1995catastrophic, tulving2002episodic} as a way to alleviate the catastrophic forgetting challenge and are thus becoming the preferred approach for continual learning models \cite{rebuffi2017icarl, aljundi2019gradient}. Unfortunately, because the memory is bounded while the number of tasks is unbounded, new and old tasks must compete for memory, and hence forgetting is inevitable. ‘How to reasonably forget’ is still a major question with significant challenges remaining, including: \textbf{1) Memory inefficiency.} The performance of the replay-based models depend heavily on the size of the available memory in the replay buffer which is used to retain as many of the previous samples as possible. While existing works typically store the entire sample into memory, we humans seldom memorize every detail of our experiences. Thus, compared to biological neural networks, some mechanisms must still be missing in current models; \textbf{2) Insufficient generalization power.} The major focus of existing works is to avoid (catastrophic) forgetting by memorizing all the details without taking into account their usefulness for learning tasks. They typically rely on episodic memory for individual tasks without sufficient chaining to make the knowledge they learn truly generalizable to all potential (historical and future unseen) tasks. In contrast, human beings significantly improve generalizability during continual learning; \textbf{3) Vulnerability to noise and corruption.} Noise and corruption are ubiquitous in real-world data and are especially likely to be present in continual learning problems. However, continual learning under these conditions has not yet been thoroughly explored. Without sufficient consideration of noise and data corruption, existing models are very vulnerable to their effects during the learning process, especially for DNNs as these have a high capacity to fit noisy labels \cite{arpit2017closer, zhang2016understanding}. 

To address these challenges, this paper proposes a novel ScheMAtic memory peRsistence and Transience (SMART) framework for continual learning with external memory based on recent advances in neuroscience. Compared to the research on memory persistence (i.e., remembering), the mechanisms responsible for memory transience (i.e., forgetting) have historically been under-explored in neuroscience until very recently \cite{richards2017persistence}. Current research on forgetting has shown that forgetting is not a ‘failure’, but rather a dedicated mechanism to facilitate mnemonic processing and generalizability. ‘Forgetting’ can be executed at various time scales and may be either active or passive, which 1) reduces memory consumption, 2) improves generalizability by discarding useless, too-specific details, and 3) builds resistance to noise and errors that cannot be consolidated.
The specific contributions presented in this paper are as follows: 
\begin{itemize}[leftmargin=15pt]
\item \textbf{Developing a new framework for efficient and robust continual learning.} 
This work aims to extend the current continual learning scenario to include more realistic and challenging settings that include noisy, irrelevant features, and data corruption. Beyond merely “avoid forgetting”, our new model leverages “passive and active forgetting” as well as “long-term and short-term forgetting”, building on the latest advances in neural-theoretical science.

\item \textbf{Proposing a novel schematic-memory-driven long-term forgetting mechanism.} To actively enforce reasonable long-term forgetting, we go beyond existing work based on “episodic memory” to “schematic memory” \cite{nelson1989remembering, sprenger1999learning} by identifying and storing generalizable knowledge across learning tasks with sparsity and ‘backward positive transfer’ constraints. The theoretical guarantee on the error bound of such a mechanism has also been analyzed.
    
\item \textbf{Achieving an effective short-term forgetting mechanism via novel neural correlation consolidation.} We construct a novel short-term forgetting mechanism that efficiently gates undesired information such as noise and corrupted data in real-time. To do this, a new computationally-efficient regularization is proposed that is inspired by background information-gated learning.
    
\item \textbf{Conducting comprehensive experiments to validate the effectiveness and efficiency of our proposed model.} Extensive experiments on one benchmark dataset and five real-world datasets demonstrate that the proposed model outperforms other comparison methods. A robust learning simulation, ablation studies, and case studies further demonstrate how the proposed components contribute to the robustness of the model against label corruption while at the same time ensuring the efficient use of memory.
\end{itemize}

\section{Related work}

\subsection{Continual learning with experience replay}

The central problem of continual learning is to overcome the catastrophic forgetting problem of neural networks. 
Nowadays experience replay has been shown to be the most effective method for mitigating catastrophic forgetting \cite{lopez2017gradient, rebuffi2017icarl, castro2018end, hou2019learning, wu2019large, hayes2020remind, chaudhry2018efficient, nguyen2017variational, riemer2018learning, aljundi2019gradient}. 
Specifically, replay-based methods alleviate the forgetting of deep neural networks by replaying stored samples from the previous history when learning new ones.
There are two main directions on how to leverage the exemplars in the memory to mitigate forgetting.
The first direction is first proposed by iCaRL \cite{rebuffi2017icarl}, which uses the knowledge distillation technique to prevent forgetting. Several other works have followed this direction including the End-to-End Incremental Learning (EEIL) \cite{castro2018end}, the Unified classifier \cite{hou2019learning}, and the Bias Correction (BiC) \cite{wu2019large}. 
More recently, REMIND \cite{hayes2020remind} proposed to store and replay mid-level representations as a more effective strategy to mitigate forgetting. However, the replay at mid-level representation requires the freeze of low-level layers and a non-trivial amount of additional data for model pre-training.
The second direction is first proposed by Gradient Episodic Memory (GEM) \cite{lopez2017gradient}, which treats the gradient of the samples in the replay buffer as extra constraints to the model when acquiring new knowledge during the course of learning. The Averaged Gradient Episodic Memory (A-GEM) \cite{chaudhry2018efficient} model further extended GEM and made the constraint computationally more efficient. 
Besides, several other directions have also been explored, such as Variational Continual Learning \cite{nguyen2017variational} that combines Bayesian inference with replay, while the Meta-Experience Replay model \cite{riemer2018learning} combines replay with meta-learning. 

However, the above works are typically built upon the assumption that the number of tasks and task boundaries are accessible before the learning start and will be used as additional information to divide the storage resource allocation for each task.
This assumption is generally not available for many real-world applications. 
To overcome this issue, recently the boundary-free continual learning has been proposed and deals with the situation where task boundary and i.i.d. (Independent and identically distributed) assumption are not available.
Several works have focused on proposing different strategies for sample selections, including reservoir sampling \cite{chaudhry2019continual, isele2018selective}, and gradient-based sample selection \cite{aljundi2019gradient}. 
Here our work is closely related to the boundary-free continual learning problem, with additional challenges of modeling robustness where the training dataset may contain a certain amount of corrupted samples.



\subsection{Episodic and Schematic memory}
Episodic memory is a neurocognitive system that enables human beings to remember past experiences \cite{tulving2002episodic}.
Recently, episodic memory has been widely adopted in rehearsal based models to overcome catastrophic forgetting challenge in continual learning \cite{lopez2017gradient, aljundi2019gradient}.
However, the bottleneck of the rehearsal based models has been highly connected to the size of the memory (i.e. replay buffer) which is used to keep previous samples as much as possible. Storing the instance completely into the memory can be highly inefficient, especially when the size of the memory is limited \cite{gardner2013secondary}. In fact, we humans seldom memorize all the details about experiences, but rather we utilize the elements of similar experiences and efficiently integrate them together as schematic memory \cite{nelson1989remembering, sprenger1999learning}. 
Such a schema-like memory will help us better generalize and extract important information from a specific event, and efficiently remember only those that are critical and of importance to the downstream learning tasks.
Therefore, this can be a great characteristic to have in conjunction with the episodic memory to fully utilize the limited memory.

\subsection{Robust Learning with Noisy Labels}
Learning from noisy labels has long been an essential study in machine learning and AI systems \cite{angluin1988learning}. Noisy labels, which are referred to corrupted labels from the actual ground-truth labels, can inevitably degenerate the robustness of learned models, especially for deep neural networks as they have the high capacity to fit noisy labels \cite{arpit2017closer, zhang2016understanding}. To handle this issue, existing works have been focused on estimating the noise transition matrix \cite{goldberger2016training, patrini2017making},  training the model on selected samples \cite{jiang2017mentornet, malach2017decoupling, han2018co}.
However, most of the existing works focus on the off-line learning problem and little has been studied on the robustness problem in the continual learning setting. 
This might be due to the fact that most existing robust learning algorithms rely heavily on a pre-acquired small clean dataset and use it to help the model deal with future corrupted samples based on the i.i.d assumption. 
On the other hand, the pre-acquisition of a clean dataset before learning starts is almost impossible in the continual learning setting. In addition,  the data from the past may not be helpful for the model to detect the corruption in current data without i.i.d assumption.
To the best of our knowledge, we are the first one to study the challenging problem of robust continual learning and to propose a model that can solve all the challenges efficiently and effectively.

\section{Problem Formulation}
\label{sec:formulation}

In this paper, we focus on replay-based continual learning, which can be defined as an online supervised learning problem. The problem formulation and its associated challenges are as follows:

A stream of learning tasks arrives sequentially with the corresponding input-output pairs $\{X^{(0)}, Y^{(0)}\},\cdots,\{X^{(t)}, Y^{(t)}\},\cdots$, for time intervals $0,1,\cdots,t-1,t,\cdots$. Each $X^{(t)}\in\mathbb{R}^{N_t,D}$, $Y^{(t)}\in\mathbb{R}^{N_t,1}$, where $N_t$ denotes the number of samples in task $t$ and $D$ is the number of input features. The data stream can be non-stationary, with no assumptions applied to the distributions such as the i.i.d. assumption.

\textbf{Problem Formulation: }Given such a stream of tasks, continual learning is commonly formulated as a  process to train a learning agent $f:\mathbb{R}^{N_t,D}\rightarrow \mathbb{R}^{N_t,1}$ parameterized by $\theta^{(t)}$ over each task $\{X^{(t)}, Y^{(t)}\}$ sequentially over time $t\in\mathbb{N}$, by minimizing the error in each task $\ell(f(X^{(t)},\theta^{(t)}),Y^{(t)})$ for each task $t$. Moreover, such a learning agent typically has a memory with a fixed memory budget in order to avoid losing its competence in previously learned tasks, which is achieved by enforcing the \emph{positive-backwards-transfer} constraint: $\ell(f(X^{(i)},\theta^{(t)}),Y^{(i)})\le\ell(f(X^{(i)},\theta^{(t-1)}),Y^{(i)})$ over all the previous tasks $i\in[0,\cdots,t-1]$.

This is a potentially very promising open problem that is currently attracting a great deal of attention from researchers, who are focusing on how to ensure the \emph{positive-backwards-transfer} constraint with fixed memory size \cite{lopez2017gradient, aljundi2019gradient}. 
However, other significant unsolved challenges remain, including : \textbf{Challenge 1: Memory inefficiency.} Existing works can only select whole data instances without the ability to identify and discard those input features that are useless for learning tasks. This significantly limits their memory efficiency, especially for high-dimensional, sparse data. Existing work neither ensures the original \emph{positive-backwards-transfer} constraints nor knows the error bound; \textbf{Challenge 2: Weakness in improving generalizability.} Existing works focus on episodic memory without the human-like ability to abstract all the isolated memory into schematic memory \cite{nelson1989remembering, sprenger1999learning} and thus obtain continuing improvements in generalization power during continual learning; and \textbf{Challenge 3: Sensitivity to noise in continual learning} Existing works assume all data are clean without corruption, which is not realistic for real world data, making the existing models sensitive to noise.






\section{SMART model}
This section presents our proposed SMART model that addresses all the above-mentioned challenges and thus narrows the gap between human and machine continual learning. Section \ref{sec:motivation} introduces biological motivations, Sections \ref{sec:long} and \ref{sec:short} describe the mechanisms proposed to achieve long-term and short-term memory transience, and Section \ref{sec:algorithm} describe the training algorithm for the new model.

\subsection{Memory Persistence and Transience for Continual Learning}
\label{sec:motivation}

Existing works emphasize ``how to remember'', but lack a model on how to actively and gracefully ``forget'' useless details (Challenges 1 and 2) and deal with noise (Challenge 3); human beings are much better at continual learning and suffer much less from these challenges \cite{gardner2015older}. The underlying mechanism through which humans achieve this so successfully has long been unclear, but recent research on ``transience'' (i.e., forgetting) and its interaction with ``persistence'' (i.e., remembering) had led to some interesting developments in neuroscientists' understanding \cite{richards2017persistence}. 

A great deal of evidence has been collected to indicate that transience happens at different time scales, with both long-term forgetting and short-term forgetting being indispensable for human life \cite{gardner2015natural}. 
These perform different functions: long-term forgetting reduces memory consumption and encourages the consolidation by discarding unnecessary details, while short-term forgetting blocks out useless noise and spurious information. Hence, in our proposed model we aim to bridge the gap between human and machine intelligence by providing a new AI continual learning that incorporates the advantageous neuro-mechanisms in human brains, leveraging long-term forgetting to address Challenges 1 and 2 and short-term forgetting to handle Challenge 3.

To achieve long-term forgetting, we actively encourage \emph{backward positive transfer} among the learning tasks, under schematic memory \cite{nelson1989remembering, sprenger1999learning}. Going beyond the episodic memory used in existing models, which stores historical tasks individually, schematic memory allows us to store only partial details $R_i(X^{(i)})$ of each task $i$ by prioritizing the knowledge that is generalizable across tasks. Specifically, as shown in Equation \eqref{eq:loss}, monotonic-increasing constraints over the sequential tasks’ performance are enforced on samples, with pruned input features specified by a regularization term $\Omega_e {(\theta)}$.  This is described in more detail in Section \ref{sec:long}.

To achieve short-term forgetting, a novel mechanism is developed to block spurious information such as noise and corrupted data in real-time. The mechanism developed to achieve this is based on background information-gated (BIG) learning theory \cite{larkin1980expert, kintsch1988role, mainetti2015neural}. A novel neural correlation consolidation regularization $\Omega^{(t)}_r {(\theta)}$ is proposed that facilitates efficient computation. Unlike $\Omega_e {(\theta)}$, this term is task-dependent and will be updated every time the model receives a new learning task (or for every batch if the learning problem is boundary-free). More details are provided in Section \ref{sec:short}.



We can now integrate all the above terms and formulate the goal of the learning problem described in Section \ref{sec:formulation} as the following constrained optimization problem:
{
\begin{gather}
\label{eq:loss}
    \argmin\nolimits_{\theta} \mathcal{L}(X^{(t)};\theta) + \alpha \Omega_e (\theta) + \beta \Omega^{(t)}_r (\theta)\\ \notag
    \text{s.t. } 
    \mathcal{L}(R_i(X^{(i)});\theta) \leq 
    \mathcal{L}(R_i(X^{(i)});\theta^{(t-1)})
\end{gather}}
where $i=0,\cdots, t-1$. Here $t$ is the index of the current example and $i$ indexes the previous examples. $n$ is the number of examples of the current observations. $\mathcal{L}(X^{(t)};\theta)$ is the concise denotation of $\ell(f(X^{(t)};\theta), Y^{(t)})$ where $f(\cdot;\theta)$ is a model parameterized by $\theta$ and $\ell$ is the loss function.  $R_i(X^{(t)})$ is a function that returns all the entries of $X^{(t)}$ except for the set of input features $r_i$, which have all-zero weights in the first hidden layer.  



\subsection{Long-term forgetting with schematic memory}
\label{sec:long}

This section focuses on two important issues: 1) how to design $\Omega_e(\theta)$ in order to actively ``forget'' unimportant input features in $X^{(i)}$ and only retain $R_i(X^{(i)})$, and an efficient way to enforce \emph{positive backward transfer} constraints; and 2) theoretical analyses and guarantees for the error bound of the proposed efficient surrogate constraints.

%


We introduce a regularization term, denoted as $\Omega_e{(\theta)}$, that models and takes advantage of schematic memory \cite{nelson1989remembering, sprenger1999learning} to regularize the model by ensuring it only considers a subset of the input features, thus reducing the memory size requirement for storing the instance and also the computational power needed for future memory replay. Regularization techniques such as group lasso (e.g., $\ell_{2,1}$-norm) \cite{yuan2006model,alvarez2016learning} can be readily used to enforce group sparsity of the weights connecting each input feature and all the first-layer hidden neurons.

%

Since the \emph{positive backward transfer} constraints in Equation \eqref{eq:loss} are prohibitively hard to maintain efficiently, here we propose to transform them into left-hand-side of Equation \eqref{eq:relex_with_eff}, which incorporates schematic memory and theoretical guarantee on the transformation error $\varepsilon$ stated in Theorem \ref{thm:error_bound}.
\small
\begin{gather}
\label{eq:relex_with_eff}
    g_*(t,i)\cdot g(i,i,t-1)^\intercal
    \geq 0\ \ \ \implies \ \ \ \mathcal{L}(R_i(X^{(i)});\theta) \leq 
    \mathcal{L}(R_i(X^{(i)});\theta^{(t-1)})+\varepsilon
\end{gather}
\normalsize
where $\varepsilon$ is the approximation error with bound analyzed in Theorem \ref{thm:error_bound}. Denote $\tilde\theta^{(i)}\subseteq \theta^{(i)}$ as the parameter set excluding the weights of all-zero features $r_i$, so its size is $\tilde J=J-|r_i|\cdot c$ where $c$ is the number of neurons in the first layer and $J$ is the total number of parameters. 
Denote the gradient $g(k,i,s)=\partial \mathcal{L}(R_i(X^{(k)});\tilde\theta^{(s)})/ \partial\tilde \theta$ and the $j$-th partial derivative is $g_j(k,i,s)=\partial \mathcal{L}(R_i(X^{(k)});\tilde\theta^{(s)})/ \partial \tilde\theta_j$. $g_*(t,i)=\partial \mathcal{L}(R_i(X^{(t)});\tilde\theta)/ \partial\tilde \theta$. Task index $i=0,\cdots, t-1$.

\begin{theorem}
\label{thm:error_bound}
The theoretical error bound of $\varepsilon$ approaches $\\ \sum\nolimits_s\min\limits_{k\le J-|r_i|c}\max\limits_{j\ne k} \left| \lambda - \frac{g_j(t,i,s)}{g_k(t,i,s)}\cdot \lambda\right|\cdot\|g(i,i,s)\|_2$  when the gradient step $\lambda$ is sufficiently small.
\end{theorem}
\begin{proof}
Theorem \ref{thm:error_bound} can be directly proved based on Lemma \ref{lm:approximation} and Theorem \ref{thm:each_bound}.
\end{proof}
Based on the the above theorem, we can easily obtain several remarks:

\noindent\textbf{Remark 1:} If $\forall (j\le J-|r_i|\cdot c):|\lambda/g_j(t,i,s)-\lambda/g_k(t,i,s)|\rightarrow 0$, then the error $\varepsilon\rightarrow0$, which means the original constraints can be precisely enforced. 

\noindent\textbf{Remark 2:}  The larger the number of useless input features, the larger $|r_i|$ will be, and hence the smaller $\varepsilon$ will be. 
\begin{lemma}
\label{lm:approximation}
We have \scriptsize $  \mathcal{L}(R_i(X^{(i)});\theta) \approx 
    \mathcal{L}(R_i(X^{(i)});\theta^{(t-1)})+\sum\limits_{s=0}^S\sum\limits_{j}\frac{ \partial\ell(f(R_i(X^{(i)});\tilde\theta^{(s)}), Y^{(i)})}{\partial \tilde\theta_j^{(s)}} \cdot \lambda$\normalsize, which means the left-hand-side is infinitely approaching to the right-hand-side when $\lambda\rightarrow 0$.
\end{lemma}
\begin{proof}
Please refer to \ref{A1} for the detailed proof.
\end{proof}
\begin{theorem}
\label{thm:each_bound}

We have the following Equation:
\small
\begin{align}
\sum\limits_{j}\frac{ \partial\ell(f(R_i(X^{(i)});\tilde\theta^{(s)}), Y^{(i)})}{\partial \tilde\theta_j^{(s)}} \cdot \lambda\le\min_k\max_{j\ne k} \left| \lambda- \frac{g_j(t,i,s)}{g_k(t,i,s)}\cdot \lambda\right|\cdot\|g(i,i,s)\|_2
\end{align}
\normalsize
\end{theorem}
\begin{proof}
Please refer to \ref{A2} for the detailed proof.
\end{proof}

\subsection{Short-term Forgetting with Neuronal Correlation Consolidation}
\label{sec:short}

Human beings are good at learning in noisy environments by distinguishing and blocking spurious associations from real ones via short-term forgetting \cite{richards2017persistence}. One of the neuroscience theories utilized to explain this phenomenon is called background information-gated (BIG) learning theory \cite{larkin1980expert, kintsch1988role, mainetti2015neural}, where the previously learned knowledge and the real associations between the concepts allow only useful and accurate information through the gate for use in future learning. This ensures spurious co-occurrences of events and associations between concepts that are formed by random noise or mislabels will not be learned.
To achieve such a learning selectivity in DNNs, it is first necessary to identify the nodes that encode important concepts learned in the past, as well as the strong associations that are encoded in the synaptic weights between the nodes. 

It is nontrivial to embed BIG theory into DNNs. To do this, we first need to express this theory at the cell-level utilizing relevant BIG-based computational models such as \cite{mainetti2015neural}. Only then can we approximate the biological neuronal operations with artificial neurons in BNNs, including neuron importance and the correlation weights between neurons. This can be considered loosely analogous to the degree of similarity of the axonal connection pattern of biological neurons in BNNs.
Suppose the matrix $\theta_{(l+1)}\in\mathbb{R}^{N_l \times N_{l+1}}$ represents all the weights between neurons in layers $l$ and $l+1$ in the DNNs, where $N_l$ and $N_{l+1}$ represent the numbers of neurons, respectively:
{\small
\begin{gather}
\label{eq:n_correlation}
    A^{(l)}=\frac{1}{N^2_{(l+1)}}\left( 
    h(\theta_{(l+1)})\cdot h(\theta_{(l+1)})^\intercal\right)
    \odot
    \left(h(\theta_{(l+1)})\cdot h(\theta_{(l+1)})^\intercal
    \right) 
\end{gather}}
where $h(\cdot)$ is an indicator function that outputs 1 if the input is non-zero and 0 otherwise, and thus discretizes the weights in the DNNs to simulate the binary nature of the connections in BNNs.
The layer-wise neuron correlation matrix $A^{(l)} \in\mathbb{R}^{N_l \times N_l}$ is a symmetric square matrix that models all the pairwise neuron correlations in layer $l$. Each entry $A^{(l)}_{i,j}$ models the correlation between neuron $i$ and neuron $j$ in terms of the similarity of their connectivity patterns. The higher the value, the stronger the correlation between the two.

Although the discretization in $h(\cdot)$ would make the computation in Equation \eqref{eq:n_correlation} more biologically plausible, it makes the model non-differentiable and difficult to train and optimize. 
In addition, in practice, it is more common to have very small weights than true zero weights and it is therefore hard to define an appropriate threshold to establish the hard partition.
Therefore, to mitigate these, we propose a differentiable approximation of the discrete version of $h(\cdot)$ as follows \footnote{Similar to the ReLU activation function, our formulation introduces a non-differentiable point at zero; we follow the conventional setting by using the sub-gradient for model optimization.}:
{\small
\begin{equation}
\label{eq:W_hat}
h(\theta_{(l+1)})=|tanh(\theta_{(l+1)})|
\end{equation}}
where $|\cdot|$ represents the element-wise absolute operator and $tanh(\cdot)$ represents the element-wise hyperbolic tangent function.
The values of $h(\theta_{(l+1)})\in\mathbb{R}^{N_l \times N_{l+1}}$ will be positive and in the range of $[0,1)$ with the value representing the relative connectivity strength of the synapse between neurons. 

Notably, the neuron correlation matrix $A^{(l)}$ can also be interpreted as an adjacency matrix which reflects the hidden relationship between neurons within layer $l$, even though there is no actual connection between them. 
Thus, similar to the way degree centrality is used to estimate the importance of the node in network science, here we can treat $A^{(l)}$ as a weighted adjacency matrix and compute the weighted degree centrality of each neuron as its importance: $p_i^{(l)}=\sum A_{i, \cdot}^{(l)}$.

Now the neuron importance $ p_i^{(l)}$ has been computed for each neuron, we can naturally assume that the connection between two important neurons will also be important, thus the importance of $[\theta_{(l)}]_{i,j}$ can be estimated by:
$P^{(l)}_{i,j}=p_i^{(l)} * p_j^{(l+1)}$.

Adopting the methods used in existing works on lifelong learning regularization design in \cite{kirkpatrick2017overcoming, zenke2017continual}, we apply the estimated synaptic importance $P^{(l)}_{i,j}$ to regularize the changes of each parameter $[\theta_{(l)}]_{i,j}$ during the course of continual learning via a standard square loss terms as follows:
{\small
\begin{equation}
\label{eq:nac_loss}
    \Omega^{(t)}_r (\theta) = \sum\nolimits_l \sum\nolimits^{N^{(l)}}_i \sum\nolimits^{N^{(l+1)}}_j \left( P^{(l)}_{i,j} ([\theta_{(l)}]_{i,j}-[\hat \theta_{(l)}]_{i,j})^2\right) 
\end{equation}}
where $[\hat \theta_{(l)}]_{i,j}$ is the stored weight between neurons $i$ and $j$ that learned from the previous task.
\subsection{Algorithm}
\label{sec:algorithm}
Algorithm 1 summarizes our proposed SMART learning algorithm. The efficient regularization $\Omega_e (\theta)$ will enable the schematic memory to only store partial details $R_t(X^{(t)})$ of task $t$, as shown in Line 7.
The robust regularization $\Omega^{(t)}_r (\theta)$ regularizes the model parameters and protects the important weights that encode the knowledge learned from previous tasks from label corruption. When the memory is full, we adopt the GSS-Greedy Sample-Selection strategy presented in \cite{aljundi2019gradient}.

\begin{algorithm}[H]
\scriptsize
\caption{SMART algorithm}
\begin{algorithmic}[1]
  \STATE Input: $n, M, \alpha, \beta, \epsilon \leftarrow 1e^{-4}$
  \STATE Initialize: $\theta, \hat{\theta}, R, \mathcal{M}, t\leftarrow 0$
  
  \REPEAT
  
  \STATE Receive: $\{X^{(t)}, Y^{(t)}\}$
  \STATE $\{\hat{X},\hat{Y}\} \leftarrow \{(X^{(t)}, Y^{(t)}) , \mathcal{M}\}$
  \STATE $\theta^{(t)}\leftarrow\argmin_{\theta}\ell(f(\hat{X};\theta), \hat{Y}) + \alpha \Omega_e (\theta) + \beta \Omega^{(t)}_r (\theta)$
  \STATE  $\mathcal{M} \leftarrow$ $\{R_t(X^{(t)}), Y^{(t)}\} \cup \mathcal{M}$
  \IF {len$(\mathcal{M})>M$}
    \STATE $\mathcal{M}\leftarrow$\text{Sample-Selection}$(\mathcal{M}, M)$
  \ENDIF
  \STATE $\hat{\theta} \leftarrow \theta^{(t)} $
  \STATE {$t\leftarrow t+1$}

    \UNTIL {run out of observations}
\end{algorithmic}
\end{algorithm}

The feature sparsity enforced by $R_i(\cdot)$ and schematic memory regularization $\Omega_e (\theta)$ make the SMART algorithm more efficient in terms of both computation complexity and memory complexity.
Specifically, at task $t$ the degree of complexity is bounded by the number of useless features $|r_t|$, as shown in the following remarks:

\noindent\textbf{Memory complexity:} The memory complexity required to store the samples $\{R_t(X^{(t)}), Y^{(t)}\}$ into the memory is $O(N_t (D-|r_t|))$. The greater the $|r_t|$, the lower the memory complexity to store the same amount of samples.

\noindent\textbf{Computational complexity:} The computational complexity of Equation \eqref{eq:relex_with_eff} is $O(J-|r_t|\cdot c)$. The greater the number of useless features $|r_t|$, the lower the computational complexity of the constraints.

\section{Experiments}

In this section, we studied the Disjoint MNIST as the benchmark dataset along with several Civil Unrest datasets for the real-world application of continual event forecasting. All the experiments were conducted on a 64-bit machine with an Intel(R) Xeon(R) W-2155 CPU \@3.30GHz processor, 32GB memory, and an NVIDIA TITAN Xp GPU.

\textbf{Disjoint MNIST Dataset}: The MNIST Benchmark \cite{lecun1998gradient} image dataset divided into 5 tasks similar with the settings in \cite{zenke2017continual, aljundi2019gradient, Hsu18_EvalCL}. Here each task is a binary classification problem between two different digits. 1k examples were used per task for training and the results on all test examples were reported.

\textbf{Civil Unrest Datasets}: Each dataset was obtained from one Latin American country. Five datasets were used respectively for the countries: Brazil, Colombia, Mexico, Paraguay, and Venezuela. 
The tweet texts from Twitter were adopted as the model inputs.
The goal is to utilize one date input to predict whether there will be an event in the next day in different cities (i.e. tasks). 
The event forecasting results were validated against a well-established labeled event set, the Gold Standard Report \cite{Doe:2018:Misc}.
More details about how the datasets are converted into the continual learning setting can be found in \ref{A3}.

\textbf{Model settings}: To ensure a fair comparison, all the experiments used the same feed-forward neural network of 2 hidden layers and 400 neurons. The Adam optimizer \cite{kingma2014adam} was used with a learning rate of 0.0001 for all the models. 
The sparsity factor $\alpha$ introduced in the proposed method was set to 0.0005 by default and the robust regularization factor $\beta$ was set via a grid search across the range 0.0001 to 0.01, depending on the corruption ratio. 
The training batch size was set to 50 and all the models were trained over 100 iterations per batch for all the datasets.

\textbf{Comparison methods}:


\textit{Online Clustering in the gradient space} (GSS-Clust): A reply-based method with an online clustering method in the gradient space for the sample selection. The distance metric was set to the Euclidean distance and the doubling algorithm for incremental clustering described in \cite{charikar2004incremental} was adopted. 

\textit{Online Clustering in the feature space} (FSS-Clust): Similar to GSS-Clust, a reply-based method with an online clustering method in the feature space for the sample selection.

\textit{IQP Gradients} (GSS-IQP): 
A reply-based method with the surrogate proposed in \cite{aljundi2019gradient} to select samples that minimized the feasible region. 

\textit{Gradient greedy selection} (GSS-Greedy): A reply-based method with an efficient greedy selection variant proposed in \cite{aljundi2019gradient}. 




\subsection{Experiments on Disjoint MNIST Benchmark}

\begin{table}\scriptsize
  \caption{Average test accuracy on disjoint MNIST with different buffer sizes. Results are averaged over 10 different random seeds.}
  \label{tab:MNIST}
  \centering
  \begin{tabular}{c|ccccc}
  \toprule
  \diagbox[width=8em]{Method}{Buffer Size}
                &100                &200                &300                
                            &400                &500\\
   \midrule
    GSS-Clust   & 75.53 $\pm$ 1.60  &80.95 $\pm$ 0.98     &85.86 $\pm$ 2.18
                            & 88.38 $\pm$ 1.31     & 89.46 $\pm$ 1.33\\ 
                            
    FSS-Clust   & 70.72 $\pm$ 2.67  &75.13 $\pm$ 1.05     &80.48 $\pm$ 2.24
                            & 86.19 $\pm$ 2.59     & 86.12 $\pm$ 0.62\\
                            
    GSS-IQP     & 78.03 $\pm$ 1.70  &86.93 $\pm$ 1.56     &88.08 $\pm$ 0.89
                            & 90.27 $\pm$ 0.33     & 91.88 $\pm$ 0.57\\
    
    GSS-Greedy  & 68.21 $\pm$ 3.67  &77.61 $\pm$ 2.55     &87.14 $\pm$ 1.56
                            & 89.96 $\pm$ 0.97     & 91.40 $\pm$ 0.53\\
    \hline
    SMART       & \textbf{85.28 $\pm$ 1.42}  &\textbf{90.46 $\pm$ 1.25}     &\textbf{91.63 $\pm$ 0.91}
                            &\textbf{92.71 $\pm$ 1.17}      &\textbf{93.23 $\pm$ 0.49}\\
    \bottomrule
   \end{tabular}
\end{table}

\textbf{Efficient Continual Learning with the Clean Dataset}:
As shown in Table \ref{tab:MNIST}, 
we compared the model performance on the Disjoint MNIST dataset across different buffer sizes. The results are averaged over 10 different random seeds and the best results are highlighted in boldface.
Overall, our proposed SMART model outperformed all the baseline methods by a significant margin. 
Moreover, the gap between SMART and all the other baseline methods was bigger when the buffer size was extremely limited. For instance, SMART outperformed the baseline methods by 9\% - 25\% when the buffer size was 100 and by 5\% - 20\% when the buffer size was 200. This is because the proposed feature selection supported the efficient usage of memory for applications where the memory is extremely limited. 
Notably, the SMART model was able to achieve an equivalent learning capacity to those achieved by the baseline models with 200 additional samples in the buffer space.

To illustrate how the proposed feature selection enhanced the efficient usage of memory in the SMART model, we compared the samples stored in the buffer during the course of continual learning for all models, as shown in Figure \ref{fig:cs_MNIST_sample}. Although the buffer size was only fixed to store a maximum of 300 full samples, the SMART model easily exceeded the cap, storing 500 effective samples during the learning process. This means that with the same amount of storage space, the SMART model can preserve more samples and enjoy a better performance even if the storage space is extremely limited.

\begin{table}\tiny
  \caption{Average test accuracy on disjoint MNIST under different corruption ratios. The buffer size of all the models was fixed to store a maximum of 300 full samples. Results are averaged over 10 different random seeds.}
  \label{tab:MNIST_corrupt}
  \centering
  \begin{tabular}{c|c|ccccc}
  \toprule
  \diagbox[width=7em]{Model}{Ratio}
                &Clean
                &10\%                &20\%                &30\%                
                            &40\%                &50\%              \\
   \midrule
    GSS-Clust   &85.86 $\pm$ 2.18
                &68.28 $\pm$ 1.98      &65.21 $\pm$ 2.92      &60.42 $\pm$ 1.44
                            & 59.23 $\pm$ 0.87     &56.11 $\pm$ 0.89\\ 
    FSS-Clust   &80.48 $\pm$ 2.24
                &69.68 $\pm$ 1.64   &66.41 $\pm$ 2.30      &62.31 $\pm$ 1.83
                            &61.58 $\pm$ 1.12      &58.53 $\pm$ 1.09\\
    GSS-IQP     &88.08 $\pm$ 0.89
                &69.45 $\pm$ 0.53      &63.95 $\pm$ 0.87      &59.28 $\pm$ 0.97
                            & 59.55 $\pm$ 0.84     &56.76 $\pm$ 1.19\\
    
    GSS-Greedy  &87.14 $\pm$ 1.56
                &75.75 $\pm$ 0.93      &71.88 $\pm$ 1.87      &64.92 $\pm$ 1.95
                            & 62.99 $\pm$ 1.02     &59.09 $\pm$ 1.34\\
    \hline
    SMART       &\textbf{93.23 $\pm$ 0.49}
                &\textbf{85.29 $\pm$ 1.58}      &\textbf{81.35 $\pm$ 1.44}      &\textbf{73.91 $\pm$ 1.65}
                            & \textbf{70.08 $\pm$ 2.16}     &\textbf{66.34 $\pm$ 2.63}\\
    \bottomrule
  \end{tabular}
\end{table}

\begin{figure}
\centering
\includegraphics[width=0.9\linewidth]{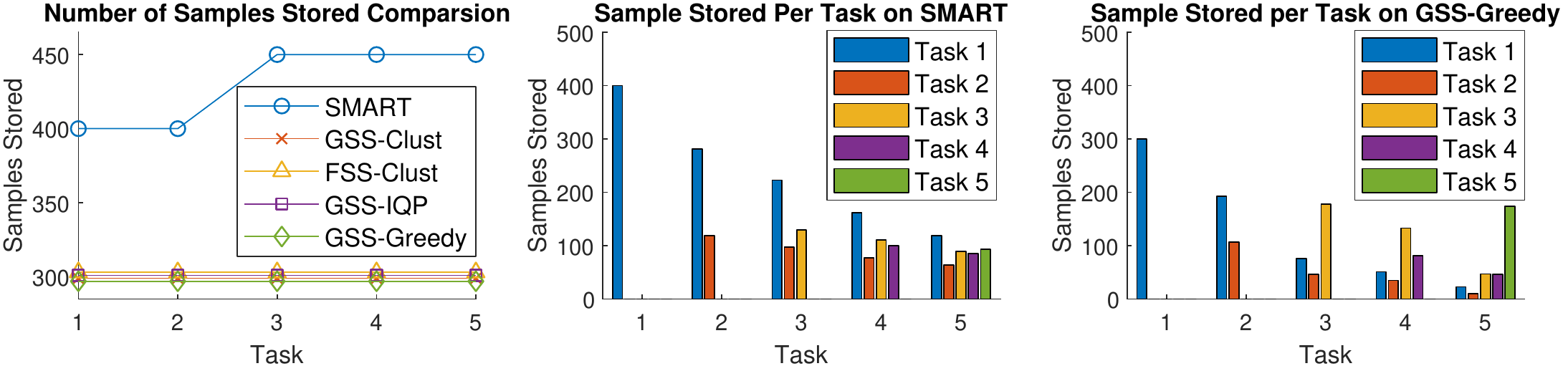}
\caption{A comparison of the total number of samples stored and the samples stored for each task in the buffer during the course of continual learning between SMART and the GSS-Greedy model. The buffer size was fixed to store a maximum of 300 full samples.}
\label{fig:cs_MNIST_sample}
\end{figure}

\textbf{Robust Continual Learning with Label Corruption}:
To further examine the robustness of our proposed SMART framework, we conducted additional experiments on the Disjoint MNIST dataset using different ratios of corrupted labels in the training samples. As shown in Table \ref{tab:MNIST_corrupt}, the SMART model once again appeared to be the most robust model for different levels of corruption. For instance, when 10\% of the training samples were corrupted, the performance of the SMART model only dropped by 8.5\% compared to the clean training data. This contrasts markedly with the performance of all the other baseline methods under the same conditions, all of which dropped significantly, ranging from 13\% - 21\% worse than their clean training counterparts. As a consequence, SMART outperformed the baseline methods even more when corruption was present, outperforming the baseline methods by as much as 15\% - 25\% under 30\% corruption. Other corruption ratios followed the same trend. 
This observation further demonstrates the superior robustness properties of our proposed robust learning regularization.



\textbf{Ablation study of Neuronal Correlation Consolidation (NCC) Regularization}:
We also conducted an ablation study of the proposed NCC regularization and compared it with existing continual learning regularization methods, including L2, EWC \cite{kirkpatrick2017overcoming, schwarz2018progress}, SI \cite{zenke2017continual}, MAS \cite{aljundi2018memory}. The hyperparameter is tuned by a grid search, and the results with the best setting are reported. For a fair comparison, all comparison methods use the same neural network architecture, which is a multi-layered perceptron with two hidden layers of 400 nodes each, followed by a softmax output layer. Both hidden layers use ReLU for the activation function. The loss function is a standard cross-entropy for classification. All models are trained for 4 epochs per task with mini-batch size 128 using the Adam optimizer.

\begin{table}\scriptsize
  \caption{Average test accuracy for the robust continual learning experiments on Split MNIST datasets under different corruption ratios. The results are averaged over 10 random seed runs.}
  \label{tab:ablation_NCC_MNIST}
  \centering
  \begin{tabular}{c|c|cccc}
  
    \toprule
    \diagbox[width=12em]{Regularization}{Corruption Ratio}                   &Clean              &10\%
                            &30\%
                            &50\%
                            &70\%\\               
    \midrule
    Vanilla                 &59.21$\pm$2.04       &58.19$\pm$1.26         
                            &58.99$\pm$0.54              
                            &58.69$\pm$0.30                 
                            &57.27$\pm$0.56     \\ 
    EWC                     &58.85$\pm$2.59     &58.62$\pm$1.68         
                            &58.83$\pm$0.65        
                            &58.59$\pm$0.34        
                            &57.53$\pm$0.74      \\
    SI                      &65.76$\pm$3.09       &64.07$\pm$4.53         
                            &63.26$\pm$6.37         
                            &62.12$\pm$6.81                
                            &57.73$\pm$8.91         \\
    L2                      &66.00$\pm$3.73     &62.20$\pm$1.98         
                            &59.13$\pm$0.56         
                            &58.77$\pm$0.22                 
                            &57.57$\pm$0.77         \\
    MAS                     &68.57$\pm$6.85     &63.81$\pm$5.95         
                            &61.24$\pm$6.00              
                            &59.31$\pm$5.91               
                            &56.99$\pm$3.76        \\
    \hline
    NCC                     &\textbf{82.24$\pm$0.91}  &\textbf{76.64$\pm$1.68} 
                            &\textbf{70.45$\pm$2.29}          
                            &\textbf{65.83$\pm$2.16}         
                            &\textbf{64.21$\pm$2.94}\\
    \bottomrule
  \end{tabular}
\end{table}

As shown in Table \ref{tab:ablation_NCC_MNIST}, the overall performance of the proposed NCC regularization outperforms all baseline regularization methods by a significant margin on both clean data and data with various corruption ratios. Specifically, the model with NCC regularization outperforms the model without any regularization (i.e. Vanilla) by 12\% - 32\% and outperforms the models with other existing regularization by up to 30\%. Interestingly, the baseline model with SI regularization also achieves some level of robustness, probability because the neuroscience inspiration behind the SI algorithm also enhanced the robustness property of the model. 
However, overall it is not as effective as the proposed NCC model. Besides, the performance of the SI model is much unstable with the relatively highest standard deviations across almost all corruption ratios. A similar issue can be also observed in MAS regularization. 

Beside just comparing the performance under the same corruption ratio, we are also interested in how the performance evolved as the corruption ratio gradually increases. As we can see in Table \ref{tab:ablation_NCC_MNIST}, all the baseline models tend to converge to around 57\% accuracy as the corruption ratio continues to increase. However, the model with NCC regularization tends to converge to around 64\% accuracy in the same scenario, leaving an around 12\% performance gap. This observation demonstrates the superior robustness property of the proposed NCC regularization.

\subsection{Experiments on Continual Event Forecasting}
\begin{table*}[ht]\tiny
  \caption{Average test accuracy for the robust continual learning experiments on civil unrest datasets under different label corruption ratios. 
  The buffer size of all the models were fixed to store a maximum of 300 full samples. Results are averaged over 10 different random seeds.}
  \centering
  \label{tab:civil_corrupt}
  \begin{tabular}{c|c|c|cccc}
    \toprule

    Dataset     & \diagbox[width=7em]{Model}{Ratio} & Clean
                            & 10\%                  & 20\%      
                            & 30\%                  & 40\% \\
    \hline
    \multirow{3}{*}{Brazil}&
                
                GSS-Clust   &78.57 $\pm$ 1.43   & 71.66 $\pm$ 0.72       &65.68 $\pm$ 2.37
                                & 61.37 $\pm$ 2.01       &53.51 $\pm$ 2.81\\&
                FSS-Clust   &79.59 $\pm$ 1.84   & 72.46 $\pm$ 0.80       &65.31 $\pm$ 2.98
                                & 60.55 $\pm$ 3.16       & 53.85 $\pm$ 1.47\\&
                GSS-IQP     & 78.02 $\pm$ 0.98  & 71.75 $\pm$ 0.86       &65.87 $\pm$ 0.84
                                & 62.83 $\pm$ 1.24       & 54.15 $\pm$ 0.97\\&
                GSS-Greedy  & 77.59 $\pm$ 2.37 & 73.70 $\pm$ 0.42       & 67.84 $\pm$ 0.85 
                                & 63.27 $\pm$ 1.98       & 55.94 $\pm$ 3.72\\&
                            
                SMART     &\textbf{83.68 $\pm$ 2.10}&\textbf{76.78 $\pm$ 1.14}& \textbf{71.38 $\pm$ 0.58}
                            & \textbf{66.59 $\pm$ 0.37}     & \textbf{57.36 $\pm$ 2.18}\\
    \hline
    \multirow{3}{*}{Colombia}&
                GSS-Clust   &71.65 $\pm$ 0.88 & 66.39 $\pm$ 1.15       & 64.52 $\pm$ 4.63
                                & 58.53 $\pm$ 2.95       & 51.39 $\pm$ 1.31\\&
                FSS-Clust   &76.27 $\pm$ 2.10& 70.40 $\pm$ 0.95       & 64.97 $\pm$ 1.65
                                & 60.53 $\pm$ 1.70       & 54.15 $\pm$ 3.25\\&
                GSS-IQP     &77.23 $\pm$ 1.02& 72.05 $\pm$ 1.00       & 64.78 $\pm$ 2.42
                                & 61.37 $\pm$ 2.07       & 54.29 $\pm$ 1.79\\&
                GSS-Greedy  &75.71 $\pm$ 1.68& 71.17 $\pm$ 1.21       & 66.79 $\pm$ 1.08
                                & 61.38 $\pm$ 2.96       & 52.85 $\pm$ 0.81\\&
                
                SMART       &\textbf{81.43 $\pm$ 1.12} & \textbf{74.76 $\pm$ 1.26 }      & \textbf{70.92 $\pm$ 2.71}
                                & \textbf{65.12 $\pm$ 3.21}       & \textbf{57.93 $\pm$ 1.16}\\
    \hline
    \multirow{3}{*}{Mexico}&
                GSS-Clust   &64.51 $\pm$ 0.85& 60.94 $\pm$ 0.90      & 57.42 $\pm$ 0.98
                                & 56.19 $\pm$ 0.87      & 53.78 $\pm$ 0.74\\&
                FSS-Clust   &64.94 $\pm$ 0.98& 61.52 $\pm$ 0.96      & 58.42 $\pm$ 0.92
                                & 56.46 $\pm$ 0.62      & 54.27 $\pm$ 0.71\\&
                GSS-IQP     &66.87 $\pm$ 0.56& 63.32 $\pm$ 0.64      & 59.03 $\pm$ 0.78
                                & 56.55 $\pm$ 0.41      & 53.73 $\pm$ 0.91\\&
                GSS-Greedy  &65.45 $\pm$ 0.16& 62.12 $\pm$ 0.79      & 59.26 $\pm$ 0.52
                                & 56.73 $\pm$ 0.70       &53.57 $\pm$ 0.53\\&
                SMART       &\textbf{70.98 $\pm$ 1.54}& \textbf{67.61 $\pm$ 0.45}      & \textbf{63.48 $\pm$ 0.46}
                                & \textbf{59.34 $\pm$ 0.43}       &\textbf{56.58 $\pm$ 0.81}\\
    \hline
    \multirow{3}{*}{Paraguay}&
                GSS-Clust   &66.81 $\pm$ 2.18& 64.36 $\pm$ 1.56      & 60.16 $\pm$ 1.37
                                & 57.84 $\pm$ 1.95      & 54.27 $\pm$ 0.73\\&
                FSS-Clust   &67.01 $\pm$ 0.95& 63.84 $\pm$ 1.36      & 59.53 $\pm$ 1.24
                                & 57.64 $\pm$ 1.58      & 53.75 $\pm$ 2.07\\&
                GSS-IQP     &64.17 $\pm$ 1.94& 62.93 $\pm$ 0.91      & 58.47 $\pm$ 1.47
                                & 56.35 $\pm$ 1.73      & 52.25 $\pm$ 1.45\\&
                GSS-Greedy  &65.77 $\pm$ 1.47& 63.68 $\pm$ 0.28      & 59.61 $\pm$ 1.64
                                & 57.70 $\pm$ 1.57      & 54.40 $\pm$ 1.60\\&
                            
                SMART     & \textbf{69.14 $\pm$ 1.83}&\textbf{66.03 $\pm$ 0.66}       &\textbf{62.11 $\pm$ 2.63}
                            & \textbf{58.91 $\pm$ 0.74}       &\textbf{55.12 $\pm$ 1.06}\\
    \hline
    \multirow{3}{*}{Venezuela}&
                GSS-Clust   &66.00 $\pm$ 1.45& 63.61 $\pm$ 1.54      & 59.87 $\pm$ 0.93
                                & 56.14 $\pm$ 0.64      & 52.56 $\pm$ 0.55\\&
                FSS-Clust   &69.28 $\pm$ 1.19& 64.73 $\pm$ 0.76      & 60.57 $\pm$ 0.43
                                & 57.47 $\pm$ 0.74      & 51.73 $\pm$ 0.57\\&
                GSS-IQP     &69.67 $\pm$ 0.79& 65.46 $\pm$ 0.39      & 60.78 $\pm$ 0.88
                                & 57.18 $\pm$ 0.48      & 53.54 $\pm$ 0.37\\&
                GSS-Greedy  &69.70 $\pm$ 0.75& 66.52 $\pm$ 0.31      & 61.91 $\pm$ 0.71
                                & 57.38 $\pm$ 1.27      & 52.73 $\pm$ 0.63\\&
                SMART    & \textbf{73.82 $\pm$ 0.18}&\textbf{69.23 $\pm$ 0.68}       &\textbf{64.78 $\pm$ 0.63}
                            & \textbf{60.73 $\pm$ 1.07}      & \textbf{55.95 $\pm$ 0.75}\\
    \bottomrule
  \end{tabular}
\end{table*}

\begin{figure}
\begin{center}
    \includegraphics[width=0.48\textwidth]{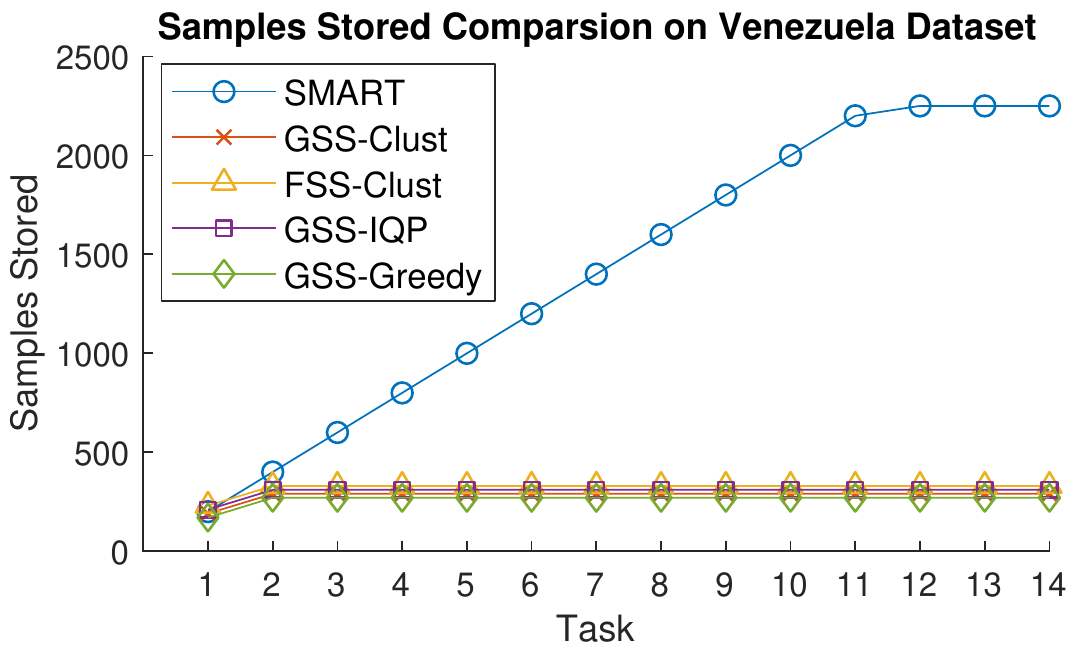}
\end{center}
\caption{A comparison of the total number of samples stored in the buffer during learning on the Venezuela Dataset. The buffer size was fixed to store at a maximum of 300 full samples. The SMART model was able to greatly exceed the cap as only a very few features were stored for each sample. Other counties followed the same trend.}
\label{fig:cs_civil_sample}
\end{figure}

\textbf{Efficient Continual Learning with the Clean Datasets}: 
The column ’Clean’ in Table \ref{tab:civil_corrupt} summaries the model performance on the five civil unrest datasets. The buffer size of all the models tested here was fixed to store a maximum of 300 full samples. 
The results were averaged over 10 different random seeds and the best results are highlighted in boldface.
Overall, our proposed SMART model achieved the best results across the datasets for all five countries, and outperforming all the baseline methods by a significant margin. On average, SMART outperformed baseline methods by 5\% - 8\% across the five datasets. 
This is because the proposed schematic memory enhanced both the efficient usage of the memory buffer and facilitated a more semantically meaningful feature selection.
To further illustrate this effect, we visualized the total number of samples stored in the buffer of the SMART model during the course of continual learning over the five Civil Unrest Datasets, as shown in Figure \ref{fig:cs_civil_sample}.  Again the buffer size is fixed to store at a maximum of 300 full samples. This time SMART is able to massively boost the number of effective samples stored, ranging from 600\% to 1000\%.
This is due to the fact that only a very few features are actually important for the prediction task and the schematic memory introduced that helps the model to gradually forget most of the non-useful features over time gradually builds up the most important features for the learning task. We conducted a case study in Section \ref{sec:case_study} and showed how the top selected features that are highly relevant to civil unrest events evolved on Brazil dataset.

\textbf{Robust Continual Learning with Label Corruption}: 
Table \ref{tab:civil_corrupt} shows the robust lifelong learning experiments on five civil unrest datasets under different label corruption ratios. Once again, the SMART model was the most robust model against different levels of corruption for all five datasets. 
This observation confirms the superior robustness property of the proposed model across various datasets and application domains.

\subsection{Comparison with Task-Aware Methods}
In this section, we compare the proposed method with State-of-the-art task-aware methods, which leverage the task boundary to decide and manage the memory allocation for samples. 

\textbf{Comparison methods}:

\textit{GEM} \cite{lopez2017gradient}: stores a fixed amount of random examples per task and uses them to provide constraints when learning new examples.

\textit{iCaRL} \cite{rebuffi2017icarl}: follows an incremental classification setting. It also stores a fixed number of examples per class but uses them to rehearse the network when learning new information.

\begin{table*}[ht]
  \caption{Performance comparison with state-of-the-art task-aware replay methods on MNIST dataset and civil unrest datasets. The average test accuracy are averaged over 10 different random seeds. The buffer size of all the models were fixed to store a maximum of 300 full samples. }
  \centering
  \label{tab:task_aware}
  \begin{tabular}{c|c|c|c}
    \toprule
    \diagbox[width=8em]{Dataset}{Model}     & iCaRL             & GEM               & SMART \\
    \midrule
    Disjoint MNIST                          & 83.27 $\pm$ 2.93  & 88.93 $\pm$ 1.07  & \textbf{91.63 $\pm$ 0.91}\\
    \hline
    Civil-Brazil                            & 73.76 $\pm$ 3.45  & 78.58 $\pm$ 0.44  & \textbf{83.68 $\pm$ 2.10}\\
    Civil-Colombia                          & 73.60 $\pm$ 1.72  & 75.29 $\pm$ 1.04  & \textbf{81.43 $\pm$ 1.12}\\
    Civil-Mexico                            & 63.11 $\pm$ 0.89  & 64.94 $\pm$ 1.11  & \textbf{70.98 $\pm$ 1.54}\\
    Civil-Paraguay                          & 63.21 $\pm$ 0.58  & 64.96 $\pm$ 1.57  & \textbf{69.14 $\pm$ 1.83}\\
    Civil-Venezuela                         & 64.61 $\pm$ 2.19  & 67.61 $\pm$ 0.49  & \textbf{73.82 $\pm$ 0.18}\\
    \bottomrule
  \end{tabular}
\end{table*}

Naturally, like other sample selection based methods, our method ignores those tasks information which can place us at a disadvantage.
Despite this disadvantage,
our proposed model was able to get a competitive or even better performance across different application domains. 
As shown in Table \ref{tab:task_aware}, the proposed model outperformed the baselines by 3\% - 10\% on MINST dataset and by 6\% - 13\% on Civil Unrest datasets, which further demonstrated the superior advantages of the efficient sample storage as well as the sample selection algorithm in the SMART model.

\subsection{Parameter Sensitivity Analysis}

\begin{figure}[htb!]
\centering
\includegraphics[width=0.8\linewidth]{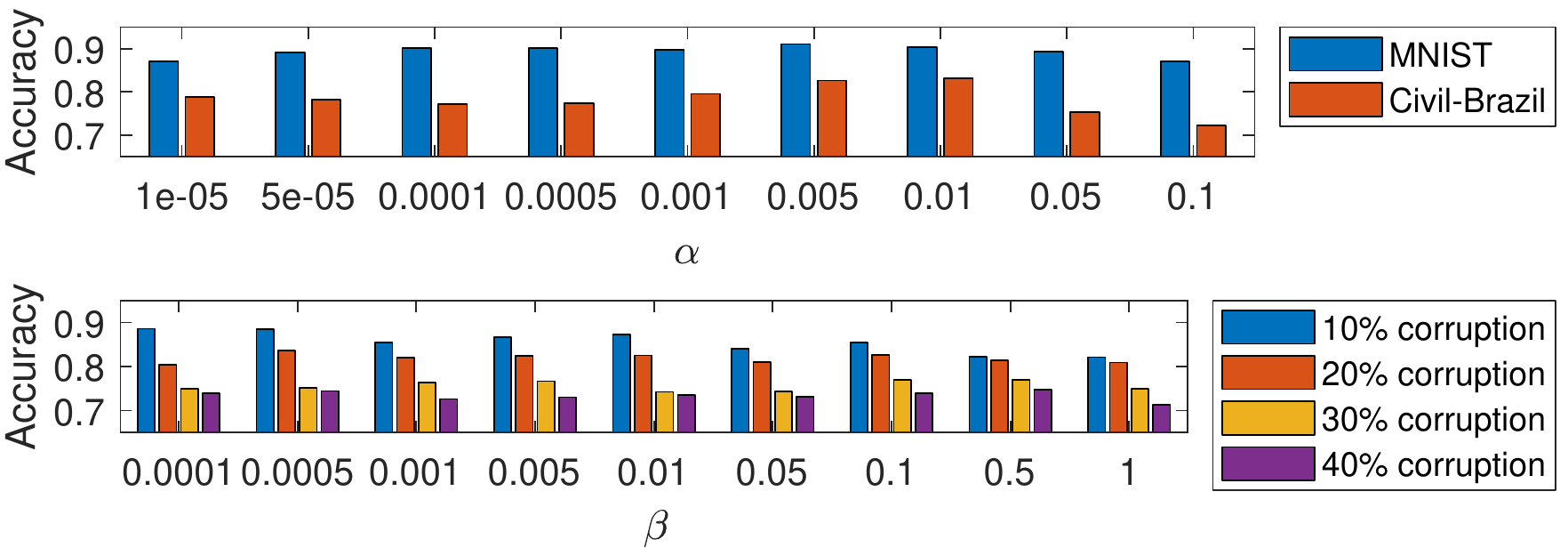}
\vspace{-10pt}
\caption{The sensitivity analysis of two regularization factors $\alpha$ and $\beta$ for the long-term and short-term forgetting regularization. The buffer size was fixed to store a maximum of 300 full samples.}
\label{fig:sensitivity}
\end{figure}

There are two hyper-parameters in the proposed SMART model, where $\alpha$ controls the schematic memory regularization and $\beta$ controls the proposed neuronal correlation consolidation regularization. 
Figure \ref{fig:sensitivity} shows the effect on the overall accuracy of the model when varying $\alpha$ and $\beta$ respectively. For $\alpha$ the results for Brazil within civil unrest datasets and MNIST dataset are shown here. As shown in the top bar chart in Figure \ref{fig:sensitivity}, by varying $\alpha$ across the range from 0.00001 to 0.1, the performance of the MNIST dataset is stable, with the fluctuation ranges less than 4\%. For the civil unrest dataset, the fluctuation range is 10\%. The best performance is obtained when $\alpha=0.005$. We can also see a clear trend where the accuracy decrease when $\alpha$ is too large or too small. The bottom bar chart illustrates the performance of the model versus $\beta$ on the MNIST dataset with various corruption ratios. The fluctuation ranges around 4-6\% on accuracy. It is also worth noting that in general when the corruption ratio increases, the $\beta$ to get the best performance also increases accordingly. Specifically, the $\beta$ that achieve the best performance are 0.0001, 0.0005, 0.1, 0.5 for corruption ratio 10\%, 20\%, 30\%, and 40\%, respectably. This demonstrated the effectiveness of the proposed robust regularization, as the model needs to rely more on the NCC regularization to be more robust against high corruption in the dataset.

\subsection{Case Study of Features (Keywords) Selected by SMART for the Brazil Civil Unrest Dataset}
\label{sec:case_study}

\begin{table}[htb!]
  \caption{Top 15 features (keywords) selected by SMART for the Brazil Civil Unrest datasets over time (all keywords have been translated to English using Google Translate). The keywords in boldface are commonly selected during the course of learning and have high correspondences with civil unrest events.}
  \label{tab:keyword_list}
        \resizebox{\linewidth}{!}{
        \centering
        \begin{tabular}{c|c|c|c|c|c|c|c}
        \toprule
         T1                     &T2                         &T3                     &T4                 &T5                    &T6                    &T7                         &T8   \\
        \midrule
         development    &    \textbf{failures}    &    \textbf{failures}    &    \textbf{failures}    &    \textbf{workplace accident}    &    essential    &    essential    &    essential \\[1ex]
 prison    &    \textbf{new law}    &    essential    &    essential    &    essential    &    accident at work    &    accident at work    &    accident at work    \\[1ex]
 damage    &    prison    &    prepare march    &    \textbf{workplace accident}    &    mother earth    &    \textbf{private property}    &    \textbf{private property}    &    \textbf{private property} \\[1ex]
prepare march    &    \textbf{private property}    &    purse    &    \textbf{private property}    &    mobilization    &    \textbf{failures}    &    \textbf{failures}    &    \textbf{nation country}\\[1ex]
 essential    &    damage    &    human chain    &    prison    &    impose    &    \textbf{new law}    &    \textbf{new law}    &    negotiations    \\[1ex]
 deforestation    &    assistants    &    \textbf{private property}    &    reforestation    &    \textbf{private property}    &    mother earth    &    \textbf{nation country}    &    \textbf{coup}    \\[1ex]
 \textbf{private property}    &    demonstration    &    prison    &    \textbf{nation country}    &    \textbf{new law}    &    reforestation    &    safety    &    \textbf{new law}   \\[1ex]
 sanctions    &    essential    &    investors    &    purse    &    dictatorship    &    \textbf{nation country}    &    reforestation    &    \textbf{criminal act}    \\[1ex]
 police operation    &    prepare march    &    \textbf{workplace accident}    &    boo    &    demand    &    alliance    &    \textbf{war crime}    &    \textbf{property rights}   \\[1ex]
 solidarity    &    \textbf{war crime}    &    dictatorship    &    prepare march    &    reforestation    &    energy production    &    \textbf{property rights}    &    \textbf{failures}\\[1ex]
 prohibition    &    human chain    &    genocide    &    human chain    &    finance    &    prepare march    &    organized    &    military action    \\[1ex]
 opportunities    &    corrupt    &    alliance    &    demand    &    security    &    \textbf{war crime}    &    \textbf{property rights}    &    energy production    \\[1ex]
 atrocity    &    property directors    &    \textbf{new law}    &    trials    &    human chain    &    incinerate    &    criminalize    &    alternatives   \\[1ex]
 pollution    &    left    &    \textbf{nation country}    &    alliance    &    alliance    &    victims    &    unemployment    &    conflict    \\[1ex]
 environmentalists    &    \textbf{coup}    &    \textbf{war crime}    &    rebel    &    \textbf{nation country}    &    necessity    &    necessity    &    \textbf{war crime}\\
        \bottomrule
        \end{tabular}
        }
\end{table}

Table \ref{tab:keyword_list} shows a case study of the top 15 features (keywords) selected by SMART for the Brazil Civil Unrest datasets over time. all keywords have been translated from Portuguese to English using Google Translate. The keywords in boldface are commonly selected during the course of learning and have high correspondences with civil unrest events. We can see a clear trend from left to right that during the course of continual learning, the SMART model was able to capture more relevant keywords that can be best used to predict civil unrest events. Thus with the help of schematic memory, the SMART model was able to gradually forget most of the non-useful features over time and gradually build up the most important features for the learning task.

\section{Conclusion}
This paper proposes a novel ScheMAtic memory peRsistence and Transience (SMART) framework for continual learning with external memory based on recent advances in neuroscience. The new framework's efficiency and generalizability are enhanced by a novel long-term forgetting mechanism and schematic memory, using sparsity and ‘backward positive transfer’ constraints with theoretical guarantee on the error bound. Furthermore, a robust enhancement is introduced that incorporates a novel short-term forgetting mechanism inspired by background information-gated (BIG) learning theory. Finally, an extensive experimental analysis on both benchmark and real-world datasets demonstrates the effectiveness and efficiency of the proposed models.

\section{Acknowledgement}
This work was supported by the NIH Grant No. R01NS39600, the NSF Grant No. 1755850, No. 1841520, No. 2007716, No. 2007976, No. 1942594, No. 1907805, a Jeffress Memorial Trust Award, Amazon Research Award, NVIDIA GPU Grant, and Design Knowledge Company (subcontract number: 10827.002.120.04)

\bibliographystyle{elsarticle-num}
\bibliography{sample-base}

\appendix

\section{Lemma 1's Proof}
\label{A1}

\renewcommand\thelemma{1}
\begin{lemma}
$\mathcal{L}(R_i(X^{(i)});\theta) \approx 
    \mathcal{L}(R_i(X^{(i)});\theta^{(t-1)})+\sum\limits_{s=0}^S\sum\limits_{j}\frac{ \partial\ell(f(R_i(X^{(i)});\tilde\theta^{(s)}), Y^{(i)})}{\partial \tilde\theta_j^{(s)}} \cdot \lambda$, the LHS is inifinitely approaching to the RHS when $\lambda\rightarrow 0$.
\end{lemma}
\begin{proof}
Define $\theta^{(t,0)}\equiv\theta^{(t-1)}$ and $\theta\equiv\theta^{(t,S)}$, we have:
\scriptsize
\begin{align}
\ell(f(R_i(X^{(i)});\theta), Y^{(i)}) &\approx \ell(f(R_i(X^{(i)});\theta^{(t,S-1)}), Y^{(i)})+\sum\limits_{j}\frac{ \partial\ell(f(R_i(X^{(i)});\theta^{(S-1)}), Y^{(i)})}{\partial \theta_j^{(S-1)}} \cdot \lambda\\
&\approx \ell(f(R_i(X^{(i)});\theta^{(t,S-2)}), Y^{(i)})+\sum\limits_{s={S-2}}^{S-1}\sum\limits_{j}\frac{ \partial\ell(f(R_i(X^{(i)});\theta^{(s)}), Y^{(i)})}{\partial \theta_j^{(s)}} \cdot \lambda\\
&\cdots\\
&\approx \ell(f(R_i(X^{(i)});\theta^{(t,1)}), Y^{(i)})+\sum\limits_{s=1}^S\sum\limits_{j}\frac{ \partial\ell(f(R_i(X^{(i)});\theta^{(s)}), Y^{(i)})}{\partial \theta_j^{(s)}} \cdot \lambda\\
&\approx \ell(f(R_i(X^{(i)});\theta^{(t,0)}), Y^{(i)})+\sum\limits_{s=0}^S\sum\limits_{j}\frac{ \partial\ell(f(R_i(X^{(i)});\theta^{(s)}), Y^{(i)})}{\partial \theta_j^{(s)}} \cdot \lambda\\
&= \ell(f(R_i(X^{(i)});\theta^{(t-1))}), Y^{(i)})+\sum\limits_{s=0}^S\sum\limits_{j}\frac{ \partial\ell(f(R_i(X^{(i)});\theta^{(s)}), Y^{(i)})}{\partial \theta_j^{(s)}} \cdot \lambda\\
&= \ell(f(R_i(X^{(i)});\theta^{(t-1))}), Y^{(i)})+\sum\limits_{s=0}^S\sum\limits_{j}\frac{ \partial\ell(f(R_i(X^{(i)});\tilde\theta^{(s)}), Y^{(i)})}{\partial \tilde\theta_j^{(s)}} \cdot \lambda
\end{align}
\normalsize
and we have 
\begin{equation}
\scriptsize
\ell(f(R_i(X^{(i)});\theta), Y^{(i)})= \lim\limits_{\lambda\rightarrow 0}\left(\ell(f(R_i(X^{(i)});\theta^{(t-1))}), Y^{(i)})+\sum\limits_{s=0}^S\sum\limits_{j}\frac{ \partial\ell(f(R_i(X^{(i)});\tilde\theta^{(s)}), Y^{(i)})}{\partial \tilde\theta_j^{(s)}} \cdot \lambda\right)
\end{equation}
where $\theta=\theta^{(t-1)}+S\cdot \lambda$.
\end{proof}

\section{Theorem 2's Proof}
\label{A2}
Since the \emph{positive backward transfer} constraints are prohibitively hard to maintain efficiently, here we propose to transform them into left-hand-side of Equation \eqref{eq:relex_with_eff}, which incorporates schematic memory and theoretical guarantee on the transformation error stated in Theorem 1.
\small
\begin{gather}
\label{eq:relex_with_eff}
    g_*(t,i)\cdot g(i,i,t-1)^\intercal
    \geq 0\ \ \ \implies \ \ \ \mathcal{L}(R_i(X^{(i)});\theta) \leq 
    \mathcal{L}(R_i(X^{(i)});\theta^{(t-1)})+\varepsilon
\end{gather}
\normalsize

\renewcommand\thetheorem{2}
\begin{theorem}
\label{thm:each_bound}

We have the following Equation:
{\small
\begin{align}
\sum\limits_{j}\frac{ \partial\ell(f(R_i(X^{(i)});\tilde\theta^{(s)}), Y^{(i)})}{\partial \tilde\theta_j^{(s)}} \cdot \lambda\le\min_k\max_{j\ne k} \left| \lambda- \frac{g_j(t,i,s)}{g_k(t,i,s)}\cdot \lambda\right|\cdot\|g(i,i,s)\|_2
\end{align}}
\end{theorem}
\begin{proof}
Then Equation \eqref{eq:relex_with_eff} can be rewritten as: $ \sum\nolimits_{j}g_j(t,i,s)g_j(i,i,s)\geq 0$. Also, suppose the step of the optimization in $t$ is small enough, then we have $g_j(t,i,s)\le 0$.
Defining a positive scalar $\xi\ge 0$ such that $\sum\nolimits_{j}g_j(t,i,s)g_j(i,i,s)-\xi= 0$, we obtain:
{\small \begin{align}
&\sum\limits_{j}\frac{ \partial\ell(f(R_i(X^{(i)});\tilde\theta), Y^{(i)})}{\partial \tilde\theta_j} \cdot \lambda\\
=&\sum_j g_j(i,i,s)\cdot \lambda\\
=&\sum_{j\ne k} (\lambda- \frac{g_j(t,i,s)}{g_k(t,i,s)}\cdot \lambda) g_j(i,i,s)+\frac{\xi}{g_k(t,i,s)}\ \ \ \ (\sum\nolimits_{j}g_j(t,i,s)g_j(i,i,s)-\xi= 0)\\
\le&\sum_{j\ne k} (\lambda- \frac{g_j(t,i,s)}{g_k(t,i,s)}\cdot \lambda) g_j(i,i,s)\ \ \ \ (\xi\ge 0,\ g_k(t,i,s)\le 0)\\
\le&\sum_{j\ne k} \left| \lambda- \frac{g_j(t,i,s)}{g_k(t,i,s)}\cdot \lambda\right|\cdot|g_j(i,i,s)|\\
=&\sum_{j\ne k} \left| \lambda- \frac{g_j(t,i,s)}{g_k(t,i,s)}\cdot \lambda\right|\cdot \frac{|g_j(i,i,s)|}{\|g(i,i,s)\|_2}\|g(i,i,s)\|_2\\
\le&\sum_{j\ne k} \left| \lambda- \frac{g_j(t,i,s)}{g_k(t,i,s)}\cdot \lambda\right|\cdot \frac{|g_j(i,i,s)|}{\|g(i,i,s)\|_1}\|g(i,i,s)\|_2\\
\le&\max_{j\ne k} \left| \lambda- \frac{g_j(t,i,s)}{g_k(t,i,s)}\cdot \lambda\right|\cdot\|g(i,i,s)\|_2\ \ \ \ \ \ (
\mbox{Holder's inequality})
\end{align}}
Since $k=1,\cdots,J$, we finally have:
{\small
\begin{align}
\sum\limits_{j}\frac{ \partial\ell(f(R_i(X^{(i)});\theta), Y^{(i)})}{\partial \theta_j} \cdot \lambda\le\min_k\max_{j\ne k} \left| \lambda- \frac{g_j(t,i,s)}{g_k(t,i,s)}\cdot \lambda\right|\cdot\|g(i,i,s)\|_2
\end{align}}
\end{proof}

\section{Continual learning experimental setup on civil unrest datasets}
\label{A3}
Each dataset was obtained from one Latin American country. Five datasets were used respectively for the countries: Brazil, Colombia, Mexico, Paraguay, and Venezuela. 
The tweet texts from Twitter were adopted as the model inputs.
The goal is to utilize one date input to predict whether there will be an event in the next day in different cities (i.e. tasks). 
In each case the data for the period from July 1, 2013 to February 9, 2014 was used for training and the data from February 10, 2014 to December 31, 2014 for the performance evaluation. 
The event forecasting results were validated against a well-established labeled event set, the Gold Standard Report \cite{Doe:2018:Misc}
\footnote{\url{Available: https://dataverse.harvard.edu/dataset.xhtml?persistentId=doi:10.7910/DVN/EN8FUW}},
GSR is a collection of civil unrest news reports from the most influential newspaper outlets in Latin America \cite{o2010tweets}. 
An example of a labeled GSR event is given by the tuple:  (City=``Maracaibo'', State =``Zulia'', Country = ``Venezuela'', Date = ``2013-01-19'', Event =``True'').

Although the original datasets were not ready to use in the context of continual learning, it can be easily converted. We first selected around 10 to 15 big cities in each country (i.e. dataset) that have high populations as the tasks. The model will need to learn to perform event forecasting in each city one by one.  To deal with the class imbalance problem, we up-sampled the positive samples (i.e. there will be an event happen the next day) and down-sampled the negative ones, yielding around 300 training samples per task. Following the online continual learning setup, the model will need to perform learning sequentially on each task's training samples in an online fashion (i.e. samples from the previous batch will not be accessible anymore unless they are stored into the replay buffer).

\end{document}